\documentclass[11pt,letter]{article}
\usepackage[papersize={8.5in,11in},margin=1in]{geometry}

\usepackage{microtype}
\usepackage{graphicx}
\usepackage{subfigure}
\usepackage{booktabs} 

\usepackage{hyperref}

\usepackage[utf8]{inputenc} 
\usepackage[T1]{fontenc}    
\usepackage{hyperref}       
\usepackage{url}            
\usepackage{booktabs}       
\usepackage{amsfonts}       
\usepackage{nicefrac}       
\usepackage{microtype}      

\usepackage{multirow}

\usepackage{times}
\usepackage{algorithm}

\usepackage[numbers,sort,compress]{natbib} 
\usepackage[noblocks]{authblk}

\usepackage{amsfonts}
\usepackage{graphicx}
\usepackage{amsmath,amsthm}
\usepackage{amssymb}
\usepackage{bm}
\usepackage{algorithmic}
\usepackage{epstopdf}
\usepackage{psfrag}
\usepackage{bbm}
\usepackage{enumerate}
\usepackage{subeqnarray}
\usepackage{cases}

\def\circlenum#1{\text{\textcircled{\textbf{#1}}}}

\newcommand{\defi}{\text{def}}

\newcommand{\argmin}{\mathop{\rm arg\min}}

\newcommand{\bbR}{\mathbb{R}}

\newcommand{\0}{{\mathbf{0}}}

\renewcommand{\a}{{\mathbf{a}}}

\renewcommand{\u}{{\mathbf{u}}}
\renewcommand{\v}{{\mathbf{v}}}

\newcommand{\x}{{\mathbf{x}}}

\newcommand{\z}{{\mathbf{z}}}

\newcommand{\w}{{\mathbf{w}}}

\newcommand{\sign}{\text{sign}}

\usepackage{color}

\usepackage{soul}

\usepackage{enumitem}
\newtheorem{theorem}{Theorem}
\newtheorem{lemma}{Lemma}

\newtheorem{proposition}{Proposition}

\newtheorem{remark}{Remark}

\begin{document}

\title{Fully Implicit Online Learning}

\author[*]{Chaobing Song}
\author[+]{Ji Liu}
\author[+]{Han Liu}
\author[*]{Yong Jiang}
\author[+]{Tong Zhang}
\affil[*]{Tsinghua University, \authorcr songcb16@mails.tsinghua.edu.cn, jiangy@sz.tsinghua.edu.cn}
\affil[+]{Tencent AI Lab, \authorcr  ji.liu.uwisc@gmail.com, hanliu@northwestern.edu, tongzhang@tongzhang-ml.org}

\maketitle

\begin{abstract}
Regularized online learning is widely used in machine learning applications.  
In online learning, performing exact minimization ($i.e.,$ implicit update) is known to 
be beneficial to the numerical stability and structure of solution. 
In this paper we study a class of regularized online algorithms without linearizing the loss function or the regularizer, which we call \emph{fully implicit online learning} (FIOL). We show that for arbitrary Bregman divergence, FIOL has the $O(\sqrt{T})$ regret for general convex setting and $O(\log T)$ regret for strongly convex setting, and the regret has an one-step improvement effect because it avoids the approximation error of linearization. Then we propose efficient algorithms to solve the subproblem of FIOL. We show that even if the solution of the subproblem has no closed form, it can be solved with complexity comparable to the linearized online algoritms. Experiments validate the proposed approaches.
\end{abstract}

\section{Introduction}
Online learning \cite{shalev2012online,hazan2016online} has a wide range of applications in recommendation, advertisement, many others.
The commonly used algorithm for online learning is online gradient descent (OGD), which linearizes the loss and regularizer in each step. OGD is simple and easy to implement. However because of linearization, OGD may incur the numerical instability issue if the step size is not properly chosen. Meanwhile it is unable to effectively explore the structure of regularizers.
To overcome the numerical stability issue of OGD, the algorithms to optimize the loss exactly ($i.e.,$ without linearization) are proposed, such as  
the well-known passive aggressive (PA) framework \cite{crammer2006online,dredze2008confidence,wang2012exact,shi2014online}, implicit online learning \cite{kulis2010implicit} and implicit SGD (I-SGD) \cite{toulis2014statistical,toulis2015scalable,toulis2016towards,toulis2017asymptotic}. To explore the structure of regularizer, the algorithms to optimize the regularizer exactly are proposed, such as composite mirror descent (COMID) \cite{duchi2010composite,duchi2011adaptive} and regularized dual averaging (RDA) \cite{xiao2010dual,chen2012optimal}. In the online setting, we call  the exact minimization to loss or regularizer as \emph{implicit update}, because it
is equivalent to OGD with an implicit step size; while the vanilla OGD is called \emph{explicit update}.

The methods that only perform implicit update with respect to ($w.r.t.$) regularizer have been well studied, such as COMID and RDA. However, 
 the analysis of the  implicit update $w.r.t.$ the loss function is proved difficult. 
In the case that the regularizer (see $r(w)$ in Table \ref{tb:alg}) does not exist , \cite{crammer2006online} gives relative loss bounds when the loss function is hinge loss or squared hinge loss.  However, the relative loss bounds are unable to be converted to a sublinear regret bound to the best of our knowledge. Then \cite{kulis2010implicit} gives the $O(\sqrt{T})$ regret bound when $f_t(\w)$ is squared loss and the $O(\log T)$ when $f_t(\w)$ is strongly convex.  The above two papers does not show any advantage of implicit update on regret bound. Meanwhile their proofs are only suitable for some particular loss functions. 

When the regularizer exists, and both the loss function and regularizer are not linearized, 
 \cite{mcmahan2010unified} 
gives the first regret bound $O(\sqrt{T})$ for general convex functions and show the one-step improvement of the implicit update. However their analysis is only suitable when the auxiliary function is Euclidean distance or Mahalanobis distance rather than arbitrary Bregman divergence. Meanwhile, \cite{mcmahan2010unified} do not give the $O(\log T)$ regret in the strongly convex setting. Moreover, the one-step improvement in \cite{mcmahan2010unified} is defined based on a constructed function, which may be counterintuitive and makes the analysis complicated. Finally, \cite{mcmahan2010unified} did not provide efficient computational methods for the nontrivial subproblem of FIOL in each iteration.

Consider the benefits of implicit update,  we study the algorithm that performs implicit update on both the loss function and regularizer, which we call fully implicit online learning (FIOL) in this paper. Compared with the theoretical analysis \cite{mcmahan2010unified} for the FIOL paradigm, we make the following improvements. First, our analysis can be applied for the general Bregman divergence, which includes Euclidean distance and Mahalanobis distance as special cases. Second, we given both $O(\sqrt{T})$ regret in the general convex setting and $O(\log{T})$ regret in the strongly convex setting. Third, we quantify the one-step improvement of implicit update as the approximation error of linearization, which makes our analysis be intuitive and  is much simpler than that of \cite{mcmahan2010unified}. 

 Meanwhile, we address the problem of solving the nontrivial subproblem of FIOL in each iteration. For the general online learning problem for empirical risk minimization, we show that the subproblem can be solved to $\epsilon$-accuracy with $O(d\log \frac{1}{\epsilon})$ by the bisection method. Then we show that for the widely used $\ell_1$-norm regularized online learning paradigm, we can solve the resulted subproblem exactly with $O(d\log d)$ cost by an deterministic algorithm and with $O(d)$ expected cost by an randomized algorithm. Experiments validate our results.


\section{Theory}
\begin{table*}[t]
\caption{The iterative procedures of online learning algorithms}
\begin{tabular}{clllllll}
\hline
Algorithm & $ $ & $ \quad(A)$ & $\quad(B)$ & $ \quad(C)$ & $ $\\
 \hline 
 SGD & $\w_{t+1}\overset{\defi}{=}\argmin_{\w\in\Omega}$ & $\Big\{\langle  f_t^{\prime}(\w_t), \w\rangle $ & $+ \langle r^{\prime}(\w_t), \w\rangle$ & $+ \frac{1}{2\eta_t}\|\w-\w_t\|_2^2\Big\}$&     \\
PA & $\w_{t+1}\overset{\defi}{=}\argmin_{\w\in\Omega}$ & $\Big\{ f_t(\w) $ & $ $ & $+ \frac{1}{2\eta_t}\|\w-\w_t\|_2^2\Big\}$ &\\
IOL  & $\w_{t+1}\overset{\defi}{=}\argmin_{\w\in\Omega}$ & $\Big\{f_t(\w) $ & $ $ & $+ \frac{1}{2\eta_t}B_{\psi}(\w, \w_t)\Big\}$ &\\
COMID  & $\w_{t+1}\overset{\defi}{=}\argmin_{\w\in\Omega}$  & $\Big\{\langle  f^{\prime}_t(\w_t), \w\rangle  $ &  $+ r(\w) $& $+ \frac{1}{2\eta_t}B_{\psi}(\w, \w_t)\Big\}$ & \\
RDA  & $\w_{t+1}\overset{\defi}{=}\argmin_{\w\in\Omega}$  & $\Big\{\frac{1}{t}\sum_{i=1}^t\langle f^{\prime}_i(\w_i), \w\rangle  $ &  $+ r(\w) $& $+ \frac{1}{2t\eta_t}\psi(\w)\Big\}$ & \\
I-SGD & $\w_{t+1}\overset{\defi}{=}\argmin_{\w\in\Omega}$ & $\Big\{f_t(\w)$  & $ +\langle r^{\prime}(\w_t), \w\rangle$ & $+ \frac{1}{2\eta_t}\|\w-\w_t\|_2^2 \Big\}$ &  \\
FIOL (\textbf{This Paper})  & $\w_{t+1}\overset{\defi}{=}\argmin_{\w\in\Omega}$ & $\Big\{f_t(\w)$  & $ +r(\w)$ & $+ \frac{1}{2\eta_t}B_{\psi}(\w, \w_t) \Big\}$ &  \\
 \hline
\end{tabular}\label{tb:alg}
\end{table*}
Before continue, we provide the notations and the problem setting first. 
Let bold italic  denote vector such as $\x\in\bbR^d$ and lower case italic denote scalar such as $x\in \bbR$. Let the Hadamard product of two vectors $\x_1$ and $\x_2$ as $\x_1\odot\x_2$.
We denote a sequence of vectors by subscripts, $i.e.,$ $\w_t,\w_{t+1}, \ldots$, and entries in a vector by non-bold subscripts, such as the $j$-th entry of $\w_t$ is $w_{tj}$. 
Let $\Omega$ denote a closed convex set in $\bbR^d$, and $\|\cdot\|_*$ denote the dual norm of  norm $\|\cdot\|$.  For a convex function $h:\Omega\rightarrow \bbR$, we use $\partial h(\w)$ to denote its subgradient set at $\w$ and use $h^{\prime}(\w)$ denote any subgradient in $\partial h(\w)$, $i.e., h^{\prime}(\w)\in \partial h(\w)$. 
 Throughout, $\psi:\Omega\rightarrow \bbR$ designates a continuously differentiable function that is $\alpha$-strongly convex w.r.t. a norm $\|\cdot\|$ on its domain $\Omega$, if for all $\w, \v\in \Omega$,
\[
\psi(\w)\ge \psi(\v) + \langle \partial \psi(\v), \w-\v\rangle + \frac{\alpha}{2}\|\w-\v\|^2.
\]
The Bregman divergence associated with $\psi(\w)$ is 
\[
B_{\psi}(\w, \v) \overset{\defi}{=} \psi(\w)-\psi(\v) - \langle  \psi^{\prime}(\v), \w-\v\rangle,
\]
which satisfies $B_{\psi}(\w, \v)\ge \frac{\alpha}{2}\|\w-\v\|^2$ for some $\alpha>0$. Finally, we assume the dataset is $\{(\x_1, y_1), (\x_2, y_2)$, $ \ldots, (\x_T, y_T)\}$, where  for all $t\in[T]$, $\x_t\in\bbR^d$ is the feature vector and $y_t\in\bbR$ is the predictive value.

In this paper we mainly consider the regularized loss minimization problem, 
 \begin{eqnarray}
  \min_{\w\in\Omega}\left\{\frac{1}{T}\sum_{t=1}^T f_t(\w) +r(\w)\right\}, \label{eq:erm}
  \end{eqnarray} 
where $\forall t\in [T], f_t:\Omega\rightarrow\bbR $ is a convex loss function, $r:\Omega\rightarrow \bbR$ is a convex regularizer, and both functions have the trivial lower bound $\forall \w\in\bbR^d, f_t(\w)\ge0, r(\w)\ge 0$.  
Examples of the above formulation include many well-known classification and regression problems. 
For binary classification, the predictive value $y_t\in\{+1,-1\}$. The  linear support vector machine (SVM) is obtained by setting $\Omega = \bbR^d, f_t(\w)=\max\{1-y_t \x_t^T \w,0\}$ and $r(\w)=\frac{\lambda}{2}\|\w\|_2^2$. For regression,  $y_t\in\bbR$. Lasso is obtained by setting $\Omega = \bbR^d, f_t(\w)=\frac{1}{2}(y_t-\x_t^T\w)^2$ and $r(\w)=\lambda\|\w\|_1$.

In online learning, Eq. \eqref{eq:erm} is optimized by a player choosing a $\w_t$ from the convex set $\Omega$ in each iteration, then a  convex loss $f_t$ is revealed and the player pays the regularized loss $f_t(\w)+r(\w)$. 
The basic task in online learning is to find an algorithm that can minimize the following regularized regret bound 
\begin{eqnarray}
R_T&\overset{\defi}{=}&\sum_{t=1}^T \left(f_t(\w_t)+r(\w_t)\right)- \min_{\w\in\Omega} \Big(\sum_{t=1}^T \left(f_t(\w) + r(\w)\right)\Big)\label{eq:regret}
\end{eqnarray}
with a sublinear rate $o(T)$. Besides the regret bound, the numerical stability and the property of solution are also of major concern.

Table \ref{tb:alg} gives the iterative procedures of some representative  online learning algorithms. In order to stabilize the iteration, all the procedures involve an auxiliary function in the $(C)$ part, where $\eta_t$ denotes the step size. SGD \cite{robbins1985stochastic}, PA \cite{crammer2006online} and I-SGD \cite{toulis2014statistical} are mainly designed for the  auxiliary function of Euclidean distance $\frac{1}{2}\|\cdot\|^2_2$, which is a special case of Bregman divergence by setting $\psi(\w)=\frac{1}{2}\|\w\|_2^2$. All other algorithms in Table \ref{tb:alg} are suitable for general $\psi(\w)$ or its Bregman divergence. 
As shown in Table \ref{tb:alg}, SGD \cite{robbins1985stochastic} linearizes both terms $f_t(\w)$ and $r(\w)$; PA \cite{crammer2006online} and IOL \cite{kulis2010implicit} perform exact minimization ($i.e,$ implicit update) on $f_t(\w)$, while do not consider the regularizer $r(\w)$;  
COMID \cite{duchi2010composite} and RDA \cite{xiao2010dual} linearize $f_t(\w)$ and perform implicit update on $r(\w)$; I-SGD linearizes $r(\w)$ and performs implicit update on $f_t(\w)$. All the above algorithms need some linearization of $f_t(\w)$ or $r(\w)$, while the FIOL algorithm 
\begin{align}
\w_{t+1}\overset{\defi}{=}\argmin_{\w\in\Omega}\Big\{f_t(\w) +r(\w)+ \frac{1}{2\eta_t}B_{\psi}(\w, \w_t) \Big\} \label{eq:double-implicit}
\end{align}
studied in this paper do not need the linearization operation. 

By using implicit update on both $f_t(\w)$ and $r(\w)$, an explicit advantage is that we get rid of the approximation error by linearization, which can be defined by 
\begin{eqnarray}
\delta_t &\overset{\defi}{=}& (f_t(\w_{t+1})+r(\w_{t+1})) - \big(f_t(\w_{t})+r(\w_t)+\langle f^{\prime}_t(\w_{t})+ r^{\prime}(\w_t), \w_{t+1}-\w_t\rangle\big).
\label{eq:one-step}
\end{eqnarray}

Because we assume that both $f_t(\w)$ and $r(\w)$ are convex, we have $\delta_t\ge 0$. 
By the definition of $\delta_t$, we obtain Lemma \ref{lem:lem1} by using a rather straightforward extension of the analysis of online mirror descent \cite{beck2003mirror}. 
\begin{lemma} \label{lem:lem1}
Let the sequence $\{\w_t\}$ be defined by the FIOL algorithm in Eq. \eqref{eq:double-implicit}. Assume that for all $t$, $f_t(\w)+r(\w)$ is convex. 
Then for any $\w\in\bbR^d$, we have
\begin{align}
&(f_t(\w_{t})+r(\w_{t}))- (f_t( \w)+r(\w))
\le\frac{1}{\eta_t}\left(B_{\psi}(\w, \w_t) -B_{\psi}(\w, \w_{t+1})\right)+\frac{\eta_t}{2\alpha}\| f^{\prime}_t(\w_t) + r^{\prime}(\w_t) \|_*^2-\delta_t.\label{eq:lem1-1}
\end{align}
\end{lemma}
Compared with the analysis of online mirror descent \cite{beck2003mirror}, the only difference is an extra term $\delta_t\ge 0$ exists on the right hand side (RHS) of \eqref{eq:lem1-1}.  
Then based on Lemma \ref{lem:lem1}, we have Theorem \ref{thm:one-step}.
\begin{theorem}\label{thm:one-step}
Let the sequence $\{\w_t\}$ be defined by the FIOL algorithm in Eq. \eqref{eq:double-implicit}.  Assume that for all $t$, $f_t(\w)+r(\w)$ is convex and $\w^*=\argmin_{\w\in\Omega} \Big(\sum_{t=1}^T \left(f_t(\w) + r(\w)\right)\Big)$.  For $t\in[T]$,
there are constants $G$ such that $\| f_t^{\prime}(\w_t)+ r^{\prime}(\w_t)\|_*\le G$ and $D$ such that $B_{\psi}(\w^*,\w_t)\le D^2$. 
Then by setting $\eta_t=\frac{\sqrt{2\alpha }D}{G\sqrt{T}}$, it follows that,
\begin{eqnarray}
R_T\le \frac{GD\sqrt{2T}}{\sqrt{\alpha}} - \sum_{t=1}^T\delta_t,\label{eq:regret1}
\end{eqnarray}

by setting $\eta_t=\frac{\sqrt{\alpha }D}{G\sqrt{t}}$, it follows that
\begin{eqnarray}
R_T\le \frac{2GD\sqrt{T}}{\sqrt{\alpha}} - \sum_{t=1}^T\delta_t,\label{eq:regret1}
\end{eqnarray}
where $R_T$ is the regularized regret defined in Eq. \eqref{eq:regret} and $\delta_t$ is the one-step improvement in Eq. \eqref{eq:one-step}.
\end{theorem}

By Theorem \ref{thm:one-step}, compared with the regret of online gradient descent \cite{hazan2016online}, FIOL has an extra gain $\sum_{t=1}^T\delta_t$, which shows the effect of FIOL that it avoids the approximation error of linearization.

Similar to online gradient descent, by assuming that for all $t,$ $f_t(\w) + r(\w)$ is $\sigma$-strongly convex $w.r.t.$ to $\psi(\w)$, that is for any $\w, \v\in \Omega$,
\begin{align}
f_t(\w) + r(\w)\ge& f_t(\v) +r(\v)+\langle 	f_t^{\prime}(\v) + r^{\prime}(\v), \w-\v\rangle + \sigma B_{\psi}(\w, \v),	
\end{align}
 we can obtain logarithmic regret for FIOL in Theorem \ref{thm:strong1}.
\begin{theorem}\label{thm:strong1}
Let the sequence $\{\w_t\}$ be defined by the FIOL algorithm in Eq. \eqref{eq:double-implicit}. Assume that for all $t$, $f_t(\w) + r(\w)$ is $\sigma$-strongly convex $w.r.t.$ $\psi(\w)$ and  $\|f_t^{\prime}(\w) + r^{\prime}(\w)\|_*\le G$. By setting $\eta_t = \frac{1}{\sigma t}$, then we have
\begin{eqnarray}
R_T\le \sigma B_{\psi}(\w, \w_1) +  \frac{G^2\log T}{2\alpha\sigma}-\sum_{t=1}^T \delta_t.
 \end{eqnarray} 
\end{theorem}

\subsection{The numerical stability of FIOL}
In Theorems \ref{thm:one-step} and \ref{thm:strong1}, where the step size is carefully chosen, the extra gain $\sum_{t=1}^{T}\delta_t$ may be small. 
However, if the step size is overlarge, in this subsection, we use a particular example to show that FIOL will not diverge, but stabilize the iteration in a fixed accuracy. 

For all $t,$ we assume $f_t(\w) \overset{\defi}{=} \phi_t(\x_t^T\w)$, where $\w\in\bbR^d$ and $\phi_t: \bbR\rightarrow \bbR$ is a $\gamma$-strongly convex function \footnote{It should be noted that the fact $\phi_t(z)$ is strongly convex about $z$ does not imply $f_t(\w)$ is strongly convex about $\w$} $w.r.t.$ $\frac{1}{2}\|\cdot\|_2^2$, $i.e.,$ for all $z, y\in\bbR$, $$\phi_t(z) \ge \phi_t(y) +\langle  \phi_t^{\prime}(y), z-y\rangle + \frac{\gamma}{2}(z-y)^2,$$
and set $\psi(\w) = \frac{1}{2}\|\w\|_2^2$. Then we have Proposition \ref{prop:1}. 
\begin{proposition}\label{prop:1}
Let the sequence $\{\w_t\}$ be defined by the FIOL algorithm in Eq. \eqref{eq:double-implicit}. Assume that for all $t$, $f_t(\w)+r(\w)$ is convex and $\eta_t = \eta_0$. 
Then for any $\w\in\bbR^d$, we have
\begin{align}
R_T
\le&\frac{1}{2\eta_0}\|\w_1-\w^*\|_2^2+r(\w_1) + \frac{\eta_0\|\x_t\|_2^2 (\phi_t^{\prime}(z))^2|_{z=\x_t^T\w_t}}{2(1+\gamma\eta_0\|\x_t\|_2^2)}.\label{eq:prop}
\end{align}
\end{proposition}

In Proposition \ref{prop:1}, we use a fixed step size $\eta_0$ in all the iterations. 
In the case that the data $\{\x_t\}$ is not normalized properly, it is possible that 
 we improperly set a large $\eta_0$ such that $\forall t, \eta_0\gg \frac{1}{\gamma\|\x_t\|_2^2}$, then 
\begin{equation}
R_T\lesssim \frac{1}{2\eta_0}\|\w_1-\w^*\|_2^2+r(\w_1)+\sum_{t=1}^T\frac{(\phi_t^{\prime}(z))^2|_{z=\x_t^T\w_t}}{2\gamma},\label{eq:R-prop} 
\end{equation}
where the second term of RHS in Eq. \eqref{eq:R-prop} is independent on $\eta_0$. Therefore, with the assumption that $\forall t, (\phi_t^{\prime}(z))^2|_{z=\x_t^T\w_t}$ is bounded by a constant, even if $\eta_0\rightarrow +\infty$,  we can still obtain an $O(T)$ regret. In contrast, in the COMID  algorithms, 
the regret will be 
\begin{equation}
R_T\le \frac{1}{2\eta_0}\|\w_1-\w^*\|_2^2+r(\w_1)+\sum_{t=1}^T\frac{\eta_0}{2}{\|f_t^{\prime}(\x_t^T\w)\|_2^2},\label{eq:R-prop} 
\end{equation}
when we use a fixed large step size $\eta_0$, the regret will be $O(\eta_0 T)$. Therefore by this regret analysis, the regret of COMID can not be guaranteed to be independent from $\eta_0$ and will be unbounded as $\eta_0\rightarrow +\infty$. Thus COMID may be unstable for overlarge $\eta_0$.

\section{Computation}
In this section, we consider the efficient computation methods to solve the subproblem of FIOL in each iteration. 
Particularly, we consider the empirical risk minimization problem and 
  assume that $\Omega\overset{\defi}{=} \bbR^d, B_{\psi}(\w, \w_t) \overset{\defi}{=}  \frac{1}{2}\|\w-\w_t\|_2^2$, $f_t(\w) \overset{\defi}{=} \phi_t(\x_t^T\w)$ and $l_t(\w) \overset{\defi}{=} r(\w) + \frac{1}{2\eta_t}\|\w-{\w_t}\|_2^2$, where $\phi_t:\bbR\rightarrow\bbR$ is a convex function and $l_t:\bbR^d\rightarrow \bbR$ is a strongly convex function.
To simplify the notation, we omit the subscript ``t'' and use $\tilde{\w}\overset{\defi}{=}\w_{t+1}$ and$\hat{\w}\overset{\defi}{=}\w_{t}$. 
Then we  rewrite the FIOL iteration as 
\begin{eqnarray}
\tilde{\w} = \argmin_{\w\in\bbR^d}\left\{\phi(\x^T\w)+ l(\w) \right\}. \label{eq:reformulate}
\end{eqnarray}
 Then assume that $\phi(z) \overset{\defi}{=} \sup_{\beta\in\bbR}\{z\beta-\phi^*(\beta)\}$ and $l^*(\z) \overset{\defi}{=} \sup_{\w\in\bbR^d}\{\z^T\w-l(\w)\}$, where $\phi(z)$ is the convex conjugate of $\phi^{*}(\beta)$ and $l^*(\z)$ is the convex conjugate of $l(\w)$. Then by the convex duality \cite{shalev2013stochastic}, we have 
\begin{eqnarray}
 &&\min_{\w\in\bbR^d}\left\{f(\w)+ r(\w) + \frac{1}{2\eta}\|\w-\hat{\w}\|_2^2\right\}	\nonumber\\
&=&\min_{\w\in\bbR^d}\left\{\phi(\x^T\w)+l(\w)\right\}\nonumber\\
&=&\min_{\w\in\bbR^d}\sup_{\beta\in\bbR}\left\{-\beta \x^T\w-\phi^*(-\beta)+l(\w)\right\}\nonumber\\
&=&\sup_{\beta\in\bbR}\left\{-\phi^*(-\beta)-\sup_{\w\in\bbR^d}(\beta \x^T\w-l(\w))\right\}\nonumber\\
&=&\sup_{\beta\in\bbR}\left\{-\phi^*( -\beta)-l^*(\beta\x)\right\}.\nonumber
\end{eqnarray}
Denote $\varphi(\beta) \overset{\defi}{=}\phi^*( -\beta)+l^*(\beta\x).$ It is known that
 if the optimal solution $\tilde{\beta}$ of $\min_{\beta\in\bbR} \varphi(\beta)$ is found, then the optimal solution of $\tilde{\w}$ is $\tilde{\w} =  \nabla l^*(\z)|_{\z = \tilde{\beta}\x}$. Therefore, the problem about $\w$ is converted to a finding the optimal solution of the one-dimensional problem $\min_{\beta\in\bbR} \varphi(\beta)$ about $\beta$, which is equivalent to finding the root of the derivative $\varphi^{\prime}(\beta)$. 
 It is known that $\varphi(\beta)$ is a convex function and thus $\varphi^{\prime}(\beta)$ is non-decreasing.  Then we can use the well-known bisection method to find an approximate root of the non-decreasing function $\varphi^{\prime}(\beta)$. In the bisection method, first we determine two points $\beta_1\in\bbR$ and $\beta_2\in\bbR$ such that $\varphi^{\prime}(\beta_1)\le 0$ and $\varphi^{\prime}(\beta_2)\ge 0$. Then we can use Alg. \ref{alg:bisection} to find an approximate root. 

\begin{algorithm}[htb]
   \caption{The bisection method}
   \label{alg:bisection}
\begin{algorithmic}[1]
	\STATE Find $\beta_1, \beta_2\in\bbR$ such that $\varphi^{\prime}(\beta_1)\le 0$ and $\varphi^{\prime}(\beta_2)\ge 0$
	\STATE ${\rm low} = \beta_1, {\rm high} = \beta_2$,  ${\rm mid} = ({\rm low}+{\rm high})/2$
    \WHILE{$|\varphi^{\prime}({\rm mid})|\ge \epsilon$}
    \STATE  ${\rm mid} = ({\rm low}+{\rm high})/2$
       \IF{$\varphi^{\prime}({\rm mid})>0$}
    \STATE ${\rm high } = { \rm mid}$
    \ELSE
    \STATE ${\rm low } = { \rm mid}$
    \ENDIF
  \ENDWHILE
   \STATE \textbf{return} mid
\end{algorithmic}
\end{algorithm}
To find an $\epsilon$-accurate root, the bisection method needs $O\Big(\log\big(\frac{\beta_2-\beta_1}{\epsilon}\big)\Big)$ iterations. In the online learning setting, evaluating $\varphi^{\prime}(\beta)$ has $O(d)$ cost in general. Therefore, to find an $\epsilon$-accurate root, the overall complexity of Alg. \ref{alg:bisection} is $O\Big(d\log\big(\frac{\beta_2-\beta_1}{\epsilon}\big)\Big)$.

For some more concrete settings, we can find better iterative algorithms or even closed-form solution. For example, if $\Omega \overset{\defi}{=}  \bbR^d, f(\w) \overset{\defi}{=} \phi(\x^T\w)  \overset{\defi}{=} \frac{1}{2}(y-\x^T\w)^2$ and $r(\w)=\frac{\lambda}{2}\|\w\|_2^2$, then by taking derivative of \eqref{eq:reformulate} directly and we can find 
\[
\tilde{\w} \overset{\defi}{=} \hat{\w} - \frac{\eta(\x^T\hat{\w}-(\eta+\lambda)y)}{(\eta+\lambda)(\eta\|\x\|_2^2+(\eta+\lambda))}\x.
\] 
In the following discussion, we consider to find the exact optimal solution in a setting which is widely used but does not have closed-form solution:  $\phi(z)$ is the convex loss function used in empirical risk minimization, such as the squared loss $\frac{1}{2}(y-z^2)$ ($y\in\bbR$), the hinge loss $\{1-yz, 0\}$  ($y\in\{-1, +1\}$), the logistic loss $\log (1+\exp(-yz))$ ($y\in\{-1, +1\}$) and  the exponential loss $\exp(-yz)$  ($y\in\{-1, +1\}$); meanwhile $r(\w) = \lambda\|\w\|_1$.

As shown in \citep[Section 5]{shalev2013stochastic}, we have 
\begin{align}
\!\!\!\!\!\!l^*(\beta\x) =& \frac{1}{2\eta} \sum_{i=1}^d \left(\max\{|\hat{w}_i+\eta \beta x_i|-\lambda\eta, 0\}\right)^2\nonumber\\
\nabla l^*(\z)|_{\z=\beta\x} =&  \sign(\hat{w}_i+\eta \beta x_i)\max\{|\hat{w}_i+\eta \beta x_i|-\lambda\eta, 0\}.\label{eq:nabla-l}
\end{align}
Define $g(\beta)\overset{\defi}{=} (l^*(\beta\x))^{\prime}$, then we have 
\begin{eqnarray}
g(\beta)\overset{\rm def}{=}&\sum_{i=1}^d  x_i\big(\max\{\hat{w}_i+\eta\beta  x_i-\lambda\eta, 0\}\nonumber\\
&\quad+\min\{\hat{w}_i+\eta\beta  x_i+\lambda\eta, 0\}\big).\label{eq:g-beta}
\end{eqnarray}
It is easy to verify that $g(\beta)$ is a piecewise linear function. 

Meanwhile, 
by the dual formulation $\phi^*(z)$ (see \citep[Section 5]{shalev2013stochastic}), we 
have Proposition \ref{prop:2}.
\begin{proposition}\label{prop:2}
For $C_1, C_2\in\bbR$, when $\phi(z)$ is square loss, hinge loss or other linear/quadratic loss, the exact root of $(\phi^*(-\beta))^{\prime}+C_1\beta+C_2$ can be found with $O(1)$ cost; when $\phi(z)$ is exponential loss or logistic loss, we can find a high-accuracy solutoin by Newton method in several $O(1)$ iterations. 
\end{proposition}

The resulted problem is to find the root of the non-decreasing function
\begin{eqnarray}
\varphi^{\prime}(\beta) =(\phi^*(-\beta))^{\prime} +g(\beta). 
\end{eqnarray}
After the optimal solution of $\tilde{\beta}$ is found, we obtain $\tilde{\w} = \nabla l^*(\z)|_{\z = \tilde{\beta}\x}$.

 In order to  find $\tilde{\beta}$, we reformulate $g(\beta)$ in Lemma \ref{lem:g-beta}.
\begin{lemma}\label{lem:g-beta}
Suppose $\u,\v\in\bbR^d$ satisfy that for all $i\in[d]$, denote $u_i \overset{\rm def}{=} -\frac{1}{\eta x_i}({\hat{w}_i- {{\rm sign}(x_i)} \lambda\eta})$, $ v_i\overset{\rm def}{=}-\frac{1}{\eta  x_i}({\hat{w}_i+{{\rm sign}(x_i)} \lambda\eta})$. Denote $\boldsymbol{\mu} \overset{\rm def}{=}[\u^T, \v^T]^T\in\bbR^{2d}, \z\overset{\rm def}{=}[(\x\odot\x)^T, -(\x\odot\x)^T]^T\in\mathbb{R}^{2d}$. Then we can rewrite $g(\beta)$ as follows
\begin{align}
\!\!g(\beta)
= \eta\sum_{i=1}^{2d}  z_i\max\{\beta-\mu_i,0\} + \eta\sum_{i=1}^dx_i^2(\beta-v_i).
 \label{eq:max-linear}
\end{align}
\end{lemma}
\begin{proof}[Proof of Lemma \ref{lem:g-beta}]
It follows that
\begin{eqnarray}
g(\beta) &\overset{\circlenum{1}}{=}& \eta\sum_{i=1}^d x_i^2\left(\max\left\{\beta-u_i,0\right\} + \min\left\{\beta-v_i,0\right\}\right)\nonumber\\
&=& \eta \left( \sum_{i: u_i< \beta}x_i^2(\beta-u_i) + \sum_{i:v_i\ge\beta}x_i^2(\beta-v_i)\right)\nonumber\\
&=&\eta\left(\sum_{i: u_i< \beta}x_i^2(\beta-u_i)  - \sum_{i:v_i<\beta}x_i^2(\beta-v_i)\right)\nonumber\\
&&+ \eta\sum_{i=1}^dx_i^2(\beta-v_i)\nonumber\\
&\overset{\circlenum{2}}{=}&\eta\sum_{i\in[2d]:\mu_i<\beta}  z_i(\beta-\mu_i) + \eta\sum_{i=1}^dx_i^2(\beta-v_i).\nonumber\\
&=& \eta\sum_{i=1}^{2d}  z_i\max\{\beta-\mu_i,0\} + \eta\sum_{i=1}^dx_i^2(\beta-v_i),\nonumber
\end{eqnarray}
where \circlenum{1} is by the definition of $\u$ and $\v$, \circlenum{2} is by the definition of $\boldsymbol{\mu}$ and $\z$. 
\end{proof}
 By \eqref{eq:max-linear}, $g(\beta)$ can be reduced to the sum of the max operators of $\beta$ plus a linear function $w.r.t.$ $\beta$.  If we know the relationship of the solution $\tilde{\beta}$ and $\mu_i (i\in[2d])$ beforehand, then $g(\beta)$ will be a linear segment and thus $\tilde{\beta}$ 
then we get a much simpler problem, which can be solved efficiently by Proposition \ref{prop:2}. Therefore, the remaining task is to determine the relationship between $\tilde{\beta}$ and $\mu_i$. In this section, we provide two kinds of algorithms: one is based on sorting; the other is based on partition.

\subsection{The sorting-based algorithm}
First we sort $\boldsymbol{\mu}$ such that $\mu_{k_1}\le\mu_{k_2}\le\cdots\le\mu_{k_{2d}}$, where $k_1,k_2,\ldots, k_{2d}$ is a permutation of $[2d]$. In addition, set $\mu_{k_0}\overset{\rm def}{=} -\infty$ and $\mu_{k_{2d+1}}\overset{\rm def}{=} +\infty$.
Then for $j\in[2d+1]$, if $\mu_{k_{j-1}}\le\beta<\mu_{k_{j}}$, then by Eq. \eqref{eq:max-linear}, we have
\begin{equation}
g(\beta) =\eta\sum_{l=1}^{j-1}z_{k_l}(\beta-\mu_{k_{l}})+ \eta\sum_{i=1}^dx_i^2(\beta-v_i), \label{eq:sort-max-linear}
\end{equation} 
which means that if we restrict $\beta$ in $\mu_{k_{j-1}}\le\beta<\mu_{k_{j}}$, then $g(\beta)$ is a linear segment. For $j\in[2d+1]$, if we compute the linear coefficients of the linear segment $g(\beta)(\mu_{k_{j-1}}\le\beta<\mu_{k_{j}})$ orderly from $j=1$ to $ 2d+1$, then we can compute all the coefficients in $O(d)$ time. If the linear coefficients of the linear segment $g(\beta)(\mu_{k_{j-1}}\le\beta<\mu_{k_{j}})$ is computed, then we can evaluate $\varphi(\beta)(\mu_{k_{j-1}}\le\beta<\mu_{k_{j}})$ in $O(1)$ time. Meanwhile if $\mu_{k_{\rho-1}}\le\tilde{\beta}\le\mu_{k_{\rho}}$, by the non-decreasing property, it must be 
\begin{eqnarray}
\varphi^{\prime}(\mu_{k_{\rho-1}})\le 0 \text{ and } 	\varphi^{\prime}(\mu_{k_{\rho}})\ge 0. \label{eq:rho}
\end{eqnarray}

Equivalently we have Lemma \ref{lem:iden}.
\begin{lemma}\label{lem:iden}
Let $\tilde{\beta}$ be the optimal solution. Let $\boldsymbol{\mu}$ and $\z$ be defined in Lemma \ref{lem:g-beta}. Let $\{k_1, k_2,\ldots, k_{2d}\}$ be the permutation of $[2d]$ such that $\mu_{k_1}\le\mu_{k_2}\le\cdots\le\mu_{k_{2d}}$ and set $\mu_{2d+1} = +\infty$. Denote $p_1\overset{\rm def}{=} \sum_{i=1}^d x_i^2, q_1\overset{\rm def}{=}\sum_{i=1}^d x_i^2 v_i$. 
Then we can find 
\begin{align}
\rho\overset{\rm def}{=}&\min\Bigg\{j\in[2d+1]: (\phi^*(-\mu_{k_j}))^{\prime}\nonumber\\
&+\eta\left(p_1+\sum_{l=1}^{j-1} z_{k_l}\right) \mu_{k_j}- \eta\left(q_1+\sum_{l=1}^{j-1} z_{k_l} \mu_{k_l}\right) \ge 0\Bigg\}\label{eq:iden}
\end{align}
such that $\tilde{\beta}$ is the solution of the equation
\begin{equation}
 (\phi^*(-\beta))^{\prime}+\eta\left(p_1+\sum_{l=1}^{\rho-1} z_{k_l}\right) \beta- \eta\left(q_1+\sum_{l=1}^{\rho-1} z_{k_l} \mu_{k_l}\right)=0. \label{eq:tilde-beta}	
\end{equation}

\end{lemma}
\begin{proof}
Because $\varphi(\beta)$ is a non-decreasing function,  the $\rho\in[2d+1]$ that satisfies Eq. \eqref{eq:rho} is equivalent to 
\begin{eqnarray}
\rho\overset{\rm def}{=}\min\left\{j\in[2d+1]:\varphi^{\prime}(\mu_{k_{\rho}})\ge0\right\}.\nonumber
\end{eqnarray} 
Then by the formulation of $g(\beta)$ in Eq. \eqref{eq:sort-max-linear} and the definitions of $p_1$ and $q_1$, we get Eq. \eqref{eq:iden}.
\end{proof}

Based on Lemma \ref{lem:iden}, we give Alg. \ref{alg:pa-1}.  


\begin{algorithm}[!ht]
    \caption{The sort-based algorithm for the FIOL problem in Eq. \eqref{eq:reformulate}}
  \begin{algorithmic}[1]\label{alg:pa-1} 
  \STATE Input: $\boldsymbol{\mu}, \z$ defined in Lemma \ref{lem:g-beta}  and three scalars $\eta>0$, $p_1\overset{\rm def}{=} \sum_{i=1}^d x_i^2, q_1\overset{\rm def}{=}\sum_{i=1}^d x_i^2 v_i$
  \STATE Sort $\boldsymbol{\mu}$ such that $\mu_{k_1}\le\mu_{k_2}\le\cdots\le\mu_{k_{2d}}$, where $\{k_1, k_2, \ldots, k_{2d}\}$ is a permutation of $[2d]$; set $\mu_{k_{2d+1}}\overset{\rm def}{=}+\infty$
  \STATE Find $\rho$ by Eq. \eqref{eq:iden}
  \STATE Set $\tilde{\beta} $ as the solution of Eq. \eqref{eq:tilde-beta}
  \STATE $\tilde{\w} =  \nabla l^*(\z)|_{\z = \tilde{\beta}\x}$ by Eq. \eqref{eq:nabla-l}
    \end{algorithmic}  
\end{algorithm}

Because a sorting operation on the vector $\boldsymbol{\mu}$ exists, Alg. \ref{alg:pa-1} has $O(d\log d)$ complexity.

\subsection{The partition-based algorithm}
According to Lemma \ref{lem:iden}, the problem to find the optimal solution $\tilde{\beta}$ is equivalent to finding the $\rho$-smallest element of $\boldsymbol{\mu}$. 
Finding the $\rho$-smallest element in a sequence is a well-known problem \cite{cormen2009introduction}, which can be solved with the $O(d)$ linear time by the randomized median algorithm \citep[\S 9]{cormen2009introduction}. Motivated by the randomized median algorithm and its variant \cite{duchi2008efficient}, in this section we propose Alg. \ref{alg:randomized-pa} to find $\tilde{\beta}$ and $\tilde{\w}$ with $O(d)$ expected time. 

In Alg. \ref{alg:randomized-pa}, we use a divide and conquer strategy to replace the sort iteration. In each iteration, according to the value of  $(\phi^*(-\mu_{k}))^{\prime}+\eta(p+\Delta p)\mu_k-\eta(q+\Delta q)$ we will determine whether to update the value of $p$ and $q$, and the set $U$ will be reduced to its subset $G$ or $L$ until $U=\emptyset$. After the loop terminates, we can obtain $\tilde{\beta}$ by finding the root of the equation in the step 12 and output $\tilde{\w}$ by Eq. \eqref{eq:nabla-l}.

To show the correctness, first we notice that after each iteration, the index $k$ of the anchor point will be removed from $U$, and thus the cardinality of $U$ will be reduced by at least $1$. Therefore, by at most $2d+2$ iterations, the loop will stop. 

Meanwhile, if we sort $\boldsymbol{\mu}$ such that $\mu_{k_1}\le \mu_{k_2}\le\cdots\le \mu_{k_{2d}}$ and set $\mu_{k_0}\overset{\defi}{=}-\infty$ and $\mu_{k_{2d+1}}\overset{\defi}{=}+\infty$,  then there must exist $i^*\in[2d+1]$ such that the optimal solution $\tilde{\beta}$ satisfies $\mu_{k_{i^*-1}}\le \tilde{\beta}\le \mu_{k_{i^*}}$. 
Then on the one hand, we reduce the cardinality of the set $U$ by the divide and conquer strategy. On the other hand, we aim to keep the following loop invariant: 
\begin{align}
h(\beta)\overset{\defi}{=}(\phi^*(-\beta))^{\prime}+\sum_{k\in {U}}z_k\max\{ \beta - \mu_k, 0\}+\eta p\beta - \eta q,  \label{eq:general}
\end{align} 
satisfies the two conditions
\begin{itemize}
\item (condition 1): $\tilde{\beta}\in [\min_{k\in {U}} \mu_k, \max_{k\in {U}} \mu_k]$ or $\mu_{k_{i^*+1}} = \min_{k\in U} \mu_k$ or $\mu_{k_{i^*}} = \max_{k\in U} \mu_k$
\item (condition 2): $h(\beta) = \varphi^{\prime}(\beta)$ if $\beta\in [\min_{k\in {U}} \mu_k, \max_{k\in {U}} \mu_k]$
\end{itemize} 
until $U=\emptyset$. After $U=\emptyset$, we can find $\tilde{\beta}$ in the step 12.

In the initialization step of Alg. \ref{alg:randomized-pa}, we initialize $U=[2d]$, $p=\sum_{i=1}^d x_i^2$, $ q= \sum_{i=1}^d x_i^2 v_i$.
Then on the one hand, because $[\mu_0 ,\min_{k\in U} \mu_k]\cup [\min_{k\in {U}} \mu_k, \max_{k\in {U}} \mu_k]\cup [\max_{k\in U} \mu_k, \mu_{2d+1}]=\bbR$, the (condition 1) is true trivially. By the definition of $p$ and $q$,  the (condition 2) is true trivially.

Then assume that before an iteration ($i.e.,$ after the previous iteration), the loop invariant holds.  
By the induction assumption, and the definition of $\Delta p$ and $\Delta q$, we have 
\begin{eqnarray}
&&\varphi(\mu_k) = h(\mu_k) \nonumber\\
&=&(\phi^*(-\mu_{k}))^{\prime}+\sum_{i\in L}z_i(\mu_k-\mu_i) + \eta p\mu_k -\eta q \nonumber\\
&=& (\phi^*(-\mu_{k}))^{\prime}+ \eta (p+\Delta p)\mu_k -\eta (q+\Delta q).
\end{eqnarray}
 Therefore in the step 5 of Alg. \ref{alg:randomized-pa}, if $\varphi(\mu_k) = (\phi^*(-\mu_{k}))^{\prime}+ \eta p\mu_k -\eta q < 0$, because $\varphi(\beta)$ is non-decreasing and we assume $\mu_{k_{i^*}}\le \tilde{\beta}\le \mu_{k_{i^*+1}}$, we have $\mu_k< \tilde{\beta}\le\mu_{k_{i^*+1}}$.  

\begin{itemize}
\item If $\tilde{\beta}\in [\min_{k\in {U}} \mu_k, \max_{k\in {U}} \mu_k]$, then we have $k_{i^*}, k_{i^*+1}\in U$. By the definition of $G$ and the condition $\mu_k< \tilde{\beta}\le\mu_{k_{i^*+1}}$, there must be $k_{i^*+1}\in G$. Therefore we have $\mu_k<\tilde{\beta}\le \max_{k\in {G}} \mu_k$, $i.e,$ $\tilde{\beta}\in [\min_{k\in {G}} \mu_k, \max_{k\in {G}} \mu_k].$

\item If $\mu_{k_{i^*+1}} = \min_{k\in U} \mu_k$, there must be $\forall k\in U, \phi(\mu_k)\ge 0$, which contradicts with our assumption that $\varphi(\mu_k)<0$. 

\item If $\mu_{k_{i^*}} = \max_{k\in U} \mu_k$, then if $G\neq \emptyset$, then by the definition of $G$ and the assumption $\mu_{k_{i^*}} = \max_{k\in U} \mu_k$, we have ${k_{i^*}} \in  G\backslash\{k\},$  there must be $\mu_{k_{i^*}}\in G$. When $G=\emptyset$,  after the iteration $U\leftarrow G\backslash\{k\}=\emptyset$, the loop stops; When $|G|\ge 2$,  

\end{itemize}
Meanwhile, for $\beta\in [\min_{k\in {G}} \mu_k, \max_{k\in {G}} \mu_k]\subset [\min_{k\in {U}} \mu_k, \max_{k\in {U}} \mu_k]$, 
\begin{eqnarray*}
\varphi(\beta)=h(\beta) &=& (\phi^*(-\beta))^{\prime}+\sum_{k\in {U}}z_k\max\{\beta - \mu_k , 0\}]\\
&&+\eta p \beta  - \eta q\\
&=&(\phi^*(-\beta))^{\prime}+ \sum_{k\in {L}}z_k\max\{\beta - \mu_k , 0\}\nonumber\\
&&+  \sum_{k\in {G}}z_k\max\{\beta - \mu_k , 0\}  +\eta p\beta - q \\
&=&(\phi^*(-\beta))^{\prime}+\sum_{k\in {G}}z_k\max\{\beta - \mu_k , 0\}\nonumber\\
 &&+ (p+\Delta p) - (q+\Delta q)\lambda.
\end{eqnarray*}

Based on the above analysis,   if $\varphi(\mu_k) = (\phi^*(-\mu_{k}))^{\prime}+ \eta p\mu_k -\eta q < 0$, by setting $p = p + \Delta p; q = q+\Delta q; U\leftarrow G\backslash\{k\}$, the loop invariant can still be true.

For the case  $(\phi^*(-\mu_{k}))^{\prime}+\eta (p+\Delta p)\mu_k - (q+\Delta q) \ge 0$, by a similar analysis,  after the update step 10, the loop invariant can still be satisfied.

By the (condition 1) and (condition 2) and the definition of $h(\beta)$, after $U=\emptyset$, we have $\varphi(\beta) =h(\beta) = (\phi^*(-\mu_{k}))^{\prime}+\eta p \beta  -\eta q$ and therefore we can find $\tilde{\beta}$ by finding the root of the equation $(\phi^*(-\mu_{k}))^{\prime}+\eta p \beta  -\eta q$.


By keeping the partial sum by $p$ and $q$, the iteration cost of Alg. \ref{alg:randomized-pa}  is $O(|U|)$. As shown in \cite{cormen2009introduction}, combined with the randomized pivot strategy, by \citep[\S 9]{cormen2009introduction} it has the expected linear time complexity $O(d)$.

\begin{remark}
If we use the median of medians strategy \cite{cormen2009introduction} to replace the randomized pivot strategy, the worst complexity of Alg. \ref{alg:randomized-pa}  will be $O(d)$. However, its empirical performance is often worse than that of the randomized pivot strategy.
\end{remark}

\begin{algorithm}[!ht]
    \caption{The partition-based algorithm for the FIOL problem in Eq. \eqref{eq:reformulate}}
\begin{algorithmic}[1]\label{alg:randomized-pa}
   \STATE Input: $\boldsymbol{\mu}$ and $\z$ defined in Lemma \ref{lem:g-beta}  and a scalar $\eta>0$ and set $ \mu_{2d+1} = +\infty$
   \STATE $p=\sum_{i=1}^d x_i^2, q= \sum_{i=1}^d x_i^2 v_i$; $U=[2d+1]$
    \vspace{0.02in}
    \WHILE{$U\neq \emptyset$}
    \STATE Pick $k\in U$
    \STATE Partition $U$:
 $\quad \quad L=\{j\in U|\mu_j\le\mu_k \}; \quad \quad G=\{j\in U|\mu_j> \mu_k\}$ 
   	\STATE Calculate $\Delta p = \sum_{j\in L}z_j;\quad$
   	 $ \Delta q = \sum_{j\in L}z_j \mu_j$
    \IF{$(\phi^*(-\mu_{k}))^{\prime}+\eta(p+\Delta p)\mu_k-\eta(q+\Delta q) < 0$}
    \STATE $p = p + \Delta p; q = q+\Delta q; U\leftarrow G$
    \ELSE
    \STATE	$U\leftarrow L\backslash\{k\}$
    \ENDIF
    \STATE Set $\tilde{\beta}$ as the solution of $(\phi^*(-\beta))^{\prime}+  \eta p\beta-\eta q=0$
    \ENDWHILE
    \STATE Output $\tilde{\w} =  \nabla l^*(\z)|_{\z = \tilde{\beta}\x}$ by Eq. \eqref{eq:nabla-l}
\end{algorithmic}  
\end{algorithm}

\section{Experiments}
\begin{table*}[th]
\centering
\caption{The best step size, the corresponding function value and sparsity}
\begin{tabular}{c|ccc|cccc}
\hline
Correlation $\rho$&\multicolumn{3}{c|}{0}&\multicolumn{3}{c}{0.5}\\
\hline
Alg& Step size   & Value &  Sparsity &  Step size  &    Value &   Sparsity    \\
\hline
SGD& $10^{-6}$ & 0.4313  &  0  & $10^{-9}$ & 0.7863 &  0 &      \\
COMID& $10^{-6}$  & 0.3964  & 100  & $10^{-9}$  & 0.7680  & 0 & \\
I-SGD& $10^{-5}$ &  0.1948 & 0  &  $10^{-4}$  & 0.163  &  0 & \\
Alg. \ref{alg:pa-1} & $10^{-4}$ & 0.3696  &  61 &  $10^{-4}$  &  0.4079  & 33\\
Alg. \ref{alg:randomized-pa}& $10^{-4}$ & 0.3049 & 111  &   $10^{-4}$  & 0.313   & 100\\
\hline
\end{tabular}\label{tb:3}
\end{table*}
In the section, to show the speed, stability and the sparsity  of solution, we compare $4$ methods: stochastic subgradient descent (SGD), online composite mirror descent (COMID), implicit SGD (I-SGD)  and the full implicit online learning in Eq. \eqref{eq:double-implicit} of this paper. Alg.  \ref{alg:pa-1} and Alg. \ref{alg:randomized-pa} are used to solve Eq. \eqref{eq:double-implicit}.  
In this experiment we solve the lasso problem 
\begin{eqnarray}
 \min_{\w\in\bbR^d}\mathbb{E}\left[1/2(\a^T\w-b)^2\right]+\lambda\|\w\|_1 \label{eq:lasso}
\end{eqnarray}
 in the online setting, where $\a$ is the sample vector, $b$ is the prediction value. 
In order to show the performance under data with different quality, following \cite{tran2015stochastic}, we use synthetic data and control the correlation coefficient betwee features. In the $t$-the iteration, a sample vector $\a_t\in\bbR^d$ is generated, where $a_{tj} = c_{tj} + \delta d_t$ with $c_{tj}\sim \mathcal{N}(0,1), d_t \sim \mathcal{N}(0,1)$ and $\delta$ is a constant. Then the correlation coefficient between $a_{i,j}$ and $a_{i,j^{\prime}}$ $(j\neq j^{\prime})$ is $\rho= \delta^2/(1+\delta^2)$. 
The prediction $b_t$ of the $t$-th iteration is defined as $b_t=\a_t^T \tilde{\w}+ \tau \epsilon_t$, where $\tilde{\w}_j=(-1)^j\exp(-2(j-1)/20)$ so that the elements of the true parameters have alternating signs and are exponentially decreasing, the noise $\epsilon_t\sim\mathcal{N}(0,1)$ and $\tau$ is chosen to control the signal-to-noise ratio.  For the $5$ algorithms, the step size is tuned over $\{10^{-10}, 10^{-9},\ldots, 10^2\}$.
We implement the $5$ algorithms in a common framework and use them to solve Eq. \eqref{eq:lasso} in the online fashion.

In this experiments, we set $d = 1000,  \tau = 0.2, \lambda=0.1, \w_1=\0$ and run all the algorithms in a fixed time under the setting $\rho = 0$ and $\rho = 0.5$.
Then the result is given in Table \ref{tb:3}.

In Table \ref{tb:3}, the column \emph{Step size} denotes the step size which makes the largest reduction of the objective function; the column \emph{Value} denote the value of objective function $\frac{1}{2N}\sum_{t=1}^N(\a_t^T\w_t-b_t)^2+\lambda\|\w_t\|_1$, where $N$ is the number of iterations; the column \emph{Sparsity} denote the number of zero elements of the solution in the last iteration.

In Table \ref{tb:3}, it is shown that the correlation between the feature vectors have large impact on the explicit update algorithm SGD and COMID which linearizes the loss function.  While the algorithms such as I-SGD, Alg. \ref{alg:pa-1} and Alg. \ref{alg:randomized-pa}, which performs implicit update for loss function, are robust for the correlation coefficient $\rho$. Because implicit update can be viewed as explicit update with data adaptive step size \cite{kulis2010implicit}, it is more robust for the scale of data and has better numerical stability.

  Meanwhile, both SGD and I-SGD linearize the regularization term $\lambda\|\w\|_1$ and thus cannot induce sparsity of solution effectively. While COMID, Alg. \ref{alg:pa-1}and Alg. \ref{alg:randomized-pa} perform implicit update for the regularization term $\lambda\|\w\|_1$. From the computational perspective, implicit update $w.r.t.$ $\lambda\|\w\|_1$  corresponds the update by soft thresholding operator, which can shrink small elements to $0$. Therefore, the $3$ algorithms have sparsity inducing effect. While it is observed that when $\rho=0.5$ and  COMID becomes unstable, it can not induce sparsity effectively. 

  Finally, under the same runtime, they can result in larger reduction of objection function than  Alg. \ref{alg:pa-1}and Alg. \ref{alg:randomized-pa} , although the iterative solving method employed by  Alg. \ref{alg:pa-1}and Alg. \ref{alg:randomized-pa}  are slower than the closed-form update of SGD and COMID. This is because that implicit update $w.r.t.$ to the loss function allows us to use a larger step size.

While because  Alg. \ref{alg:pa-1}and Alg. \ref{alg:randomized-pa}  and I-SGD can use the same step size and the closed-form update of I-SGD is faster, under the same run time, I-SGD can get a larger reduction of objection function.  
However, it should be noted that first,  to the best of our knowledge, the proposed   Alg. \ref{alg:pa-1}and Alg. \ref{alg:randomized-pa}  algorithms are the first attempts to solve the full implicit online learning problem in Eq. \eqref{eq:double-implicit} efficiently; second compared to I-SGD,  Alg. \ref{alg:pa-1}and Alg. \ref{alg:randomized-pa}  can induce sparsity effectively.

\section{Conclusion}
In this paper, we mainly study an online algorithm which perform exact minimization ($i.e,$ implicit update) for both loss function and regularizer. By performing implicit update, it avoids the approximation error of linearization, keeps the numerical stability when the step size is properly set to a large value, and exploits the structure of regularizer to obtain a structure solution. The regret bound analyses are given in given for FIOL. Meanwhile, we propose efficient computational algorithms to solve the nontrivial subproblem of FIOL, while these computational algorithms are only suitable for the empirical risk minimization (ERM) problem.  In the future, we will explore more efficient computational algorithms for the problems beyond the ERM problem. 
\vskip 0.2in
\bibliographystyle{alpha}
\bibliography{fiol.bib}

\clearpage

\appendix
\onecolumn

\begin{proof}[Proof of Lemma \ref{lem:lem1}]
In the proof, we use $\partial f_t(\x_{t}^T\w)$ to denote the subgradient $w.r.t.$ $\w$ and use $\partial f_t(z)|_{z =\x_{t}^T\w} $ to denote the subgradient $w.r.t.$ the scalar $\x_{t}^T\w$. 

For Eq. \eqref{eq:double-implicit}, the optimality condition of $\w_{t+1}$ implies $\forall \w\in\Omega$, and $f_t^{\prime}(\w_{t+1})\in \partial f_t(\w_{t+1}), r^{\prime}(\w_{t+1})\in\partial r(\w_{t+1})$,
\begin{eqnarray}
\langle \w - \w_{t+1},f_t^{\prime}(\w_{t+1})+ r^{\prime}(\w_{t+1})+\frac{1}{\eta_t}(\nabla\psi(\w_{t+1})-\nabla\psi(\w_{t}))\ge 0.\label{eq:opt}
\end{eqnarray}
Then it follows that 
\begin{eqnarray}
&&(f_t(\w_{t+1})+r(\w_{t+1}))- (f_t(\w)+r(\w))\nonumber\\
&\overset{\circlenum{1}}{\le}& \langle f_t^{\prime}(\w_{t+1}) + r^{\prime}(\w_{t+1}) , \w_{t+1}-\w\rangle\nonumber \\ 
&\overset{\circlenum{2}}{=}& \frac{1}{\eta_t}\langle\nabla \psi(\w_t)-\psi(\w_{t+1}), \w_{t+1}-\w\rangle\nonumber\\
&\overset{\circlenum{3}}{=}& \frac{1}{\eta_t}\left(B_{\psi}(\w, \w_{t})-B_{\psi}(\w, \w_{t+1})-B_{\psi}(\w_{t+1}, \w_{t})\right), \label{eq:lem1-1}
\end{eqnarray}
where \circlenum{1} is by the convexity of $f_t( \w)+r(\w)$, \circlenum{2} is by the optimality condition Eq. \eqref{eq:opt},  \circlenum{3} is by the triangle inequality.
Meanwhile
\begin{eqnarray}
&&f_t(\w_{t+1})+r(\w_{t+1}) - (f_t( \w)+r(\w))  \nonumber \\
&=& f_t(\w_{t})+r(\w_t)  - (f_t(\w)+r(\w)) +\langle f_t^{\prime}(\w_{t})+r^{\prime}(\w_t), \w_{t+1}-\w_t\rangle \nonumber \\
&&+  f_t(\w_{t+1}) + r(\w_{t+1})-(f_t(\w_{t})+r(\w_t))
-\langle  f_t^{\prime}(\w_{t})+r^{\prime}(\w_t), \w_{t+1}-\w_t\rangle \nonumber\\
&\overset{\circlenum{1}}{=}&f_t(\w_{t})+r(\w_{t})  - f_t(\w)-r(\w) +\langle f_t^{\prime}(\w_{t})+r^{\prime}(\w_t), \w_{t+1}-\w_t\rangle + \delta_t,\label{eq:lem1-2}
\end{eqnarray}
where \circlenum{1} is by the definition of $\delta_t$ in Eq. \eqref{eq:one-step}.

By Eq. \eqref{eq:lem1-1} and   \eqref{eq:lem1-2}, it follows that

\begin{eqnarray}
&&(f_t(\w_{t})+r(\w_{t}))- (f_t(\w)+r(\w))\nonumber\\
&\le&\frac{1}{\eta_t}\left(B_{\psi}(\w, \w_t) -B_{\psi}(\w, \w_{t+1})-B_{\psi}(\w_{t+1}, \w_t)\right)\nonumber\\
&&\quad-\langle f_t^{\prime}(\w_{t})+r^{\prime}(\w_t), \w_{t+1}-\w_t\rangle-\delta_t\nonumber\\
&\overset{\circlenum{1}}{\le}&\frac{1}{\eta_t}\left(B_{\psi}(\w, \w_t) -B_{\psi}(\w, \w_{t+1})\right)-\frac{\alpha}{2\eta_t}\|\w_{t+1}-\w_t\|^2\nonumber\\
&&\quad-\langle f_t^{\prime}(\w_{t})+r^{\prime}(\w_t), \w_{t+1}-\w_t\rangle-\delta_t\nonumber\\
&\overset{\circlenum{2}}{\le}& \frac{1}{\eta_t}\left(B_{\psi}(\w, \w_t) -B_{\psi}(\w, \w_{t+1})\right)+\frac{\eta_t}{2\alpha}\| f^{\prime}_t(\w_t) + r^{\prime}(\w_t) \|_*^2-\delta_t,\nonumber
\end{eqnarray}
where \circlenum{1} is by the property of Bregman divergence $B_{\psi}(\w_{t+1}, \w_t)\ge \frac{\alpha}{2\eta_t}\|\w_{t+1}-\w_t\|^2$, 
 \circlenum{2} follows from the Fenchel-Young inequality applied to $\|\cdot\|_2^2$.

Lemma \ref{lem:lem1} is proved.
\end{proof}

\begin{proof}[Proof of Theorem \ref{thm:one-step}]
It follows that
\begin{eqnarray}
&&(f_t(\w_{t})+r(\w_{t}))- (f_t(\w)+r(\w))\nonumber\\
&\overset{\circlenum{1}}{\le}& \frac{1}{\eta_t}\left(B_{\psi}(\w, \w_t) -B_{\psi}(\w, \w_{t+1})\right)+\frac{\eta_t}{2\alpha}\| f^{\prime}_t(\w_t) + r^{\prime}(\w_t) \|_*^2-\delta_t\nonumber\\
&\overset{\circlenum{2}}{\le}&   \frac{1}{\eta_t}\left(B_{\psi}(\w, \w_t) -B_{\psi}(\w, \w_{t+1})\right)+ \frac{\eta_t G^2}{2\alpha}-\delta_t,\label{eq:one-step-1}
\end{eqnarray}
where \circlenum{1} is by Lemma \ref{lem:lem1},  and \circlenum{2} is by the assumption $\|f^{\prime}_t(\w_t) + r^{\prime}(\w_t)\|_2\le G$.

In addition, by the assumption  $B_{\psi}(\w, \w_t)\le D^2$ and by setting $\eta_t = \frac{\sqrt{\alpha}D}{G\sqrt{t}}$, we have
\begin{eqnarray}
&&\sum_{t=1}^T \frac{1}{\eta_t}\left(B_{\psi}(\w, \w_t)-B_{\psi}(\w, \w_{t+1})\right)\nonumber \\
&=&\sum_{t=1}^T\left( \frac{G\sqrt{t}}{\sqrt{\alpha}D}\left(B_{\psi}(\w, \w_t)-B_{\psi}(\w, \w_{t+1})\right)\right)\nonumber  \\
&=&\sum_{t=1}^T\left( \frac{G}{\sqrt{\alpha}D}\left( \sqrt{t-1}\|\w_t-\w\|_2^2- \sqrt{t}\|\w_{t+1}-\w\|_2^2 + (\sqrt{t}-\sqrt{t-1})\|\w_t-\w\|_2^2 \right)\right)\nonumber \\
&\le&\frac{GD\sqrt{T}}{\sqrt{\alpha}}  \label{eq:one-step-2}
\end{eqnarray}
and 
\begin{eqnarray*}
\sum_{t=1}^T\frac{\eta_t  G^2}{2} = \sum_{t=1}^T\frac{\sqrt{\alpha}D  G}{2\sqrt{t}}\le\int_{t=1}^T\frac{ GD}{2\sqrt{t}} dt\le \frac{GD\sqrt{T}}{\sqrt{\alpha}}\label{eq:one-step-3}
\end{eqnarray*}
By Eq. \eqref{eq:one-step-1}, \eqref{eq:one-step-2} and \eqref{eq:one-step-2}, we have
\begin{eqnarray*}
\sum_{t=1}^T\left((f_t(\w_{t})+r(\w_{t}))- (f_t(\w)+r(\w))\right)
\le \frac{2GD\sqrt{T}}{\sqrt{\alpha}} -\sum_{t=1}^T\delta_t.\label{eq:tt4}
\end{eqnarray*}
By setting $\w^*=\argmin_{\w\in\bbR^d}\left(\sum_{t=1}^T \left(f_t(\w) + r(\w)\right)\right)$ in Eq. \eqref{eq:tt4}, Theorem \ref{thm:one-step} is proved.
\end{proof}

\begin{lemma}\label{lem:strong1}
Let the sequence $\{\w_t\}$ be defined by the FIOL algorithm in Eq. \eqref{eq:double-implicit}. Assume that for all $t$, $f_t(\w) + r(\w)$ is $\sigma$-strongly convex $w.r.t.$ $\psi(\w)$. Then we have 
\begin{eqnarray}
&&(f_t(\w_{t})+r(\w_{t}))- (f_t(\w)+r(\w))\nonumber\\
&\le&\frac{1}{\eta_t}\left(B_{\psi}(\w, \w_t) -B_{\psi}(\w, \w_{t+1})\right)+\frac{\eta_t}{2\alpha}\| f^{\prime}_t(\w_t) + r^{\prime}(\w_t) \|_*^2- \sigma B_{\psi}(\w, \w_{t+1})-\delta_t.
\end{eqnarray}

\end{lemma}

\begin{proof}[Proof of Lemma \ref{lem:strong1}]
The proof is effectively identical to that of Lemma \ref{lem:lem1}. Note that 
\begin{eqnarray}
f_t(\w_{t+1})+r(\w_{t+1}))- (f_t(\w)+r(\w)) + \sigma B_{\psi}(\w, \w_{t+1})\le \langle f_t^{\prime}(\w_{t+1}) + r^{\prime}(\w_{t+1}) , \w_{t+1}-\w\rangle\nonumber \\ 
\end{eqnarray}
Now we simply proceed as in the proof of Lemma \ref{lem:lem1}.
\end{proof}

\begin{proof}[Proof of Theorem \ref{thm:strong1}]
By Lemma \ref{lem:strong1}, it follows that 
\begin{eqnarray}
&&\sum_{t=1}^T f_t(\w_t) + r(\w_t) - f_t(\w) - r(\w)\nonumber\\
&\le&\sum_{t=1}^T \left( \frac{1}{\eta_t}\left(B_{\psi}(\w, \w_t) -B_{\psi}(\w, \w_{t+1})\right)+\frac{\eta_t}{2\alpha}\| f^{\prime}_t(\w_t) + r^{\prime}(\w_t) \|_2^2- \sigma B_{\psi}(\w, \w_{t+1})-\delta_t\right) \nonumber\\
&\le&\frac{1}{\eta_1}B_{\psi}(\w, \w_1) - \frac{1}{\eta_T}B_{\psi}(\w, \w_{t+1}) + \sum_{t=1}^{T-1}B_{\psi}(\w, \w_{t+1})\left( \left(\frac{1}{\eta_{t+1}}-\frac{1}{\eta_{t}}\right)-\sigma  \right) \nonumber\\
&&+\frac{\eta_t}{2\alpha}\| f^{\prime}_t(\w_t) + r^{\prime}(\w_t) \|_*^2 -\sum_{t=1}^T \delta_t.
\end{eqnarray}
By setting $\eta_t = \frac{1}{\sigma t}$, then we have $\frac{1}{\eta_{t+1}}-\frac{1}{\eta_{t}} - \sigma =0$. By assuming that $\| f^{\prime}_t(\w_t) + r^{\prime}(\w_t) \|_2\le G$, we have
\begin{eqnarray}
&&\sum_{t=1}^T f_t(\w_t) + r(\w_t) - f_t(\w) - r(\w)\nonumber\\
&\le& \sigma B_{\psi}(\w, \w_1) + \frac{G^2}{2\alpha}\sum_{t=1}^T\eta_t-\sum_{t=1}^T \delta_t\nonumber\\
&\le& \sigma B_{\psi}(\w, \w_1) +  \frac{G^2\log T}{2\alpha\sigma}-\sum_{t=1}^T \delta_t\label{eq:ttt5}
\end{eqnarray}
By setting $\w^*=\argmin_{\w\in\bbR^d}\left(\sum_{t=1}^T \left(f_t(\w) + r(\w)\right)\right)$ in Eq. \eqref{eq:ttt5}, Theorem \ref{thm:strong1} is proved.

\end{proof}

\begin{proof}[Proof of Proposition \ref{prop:1}]

By the  assumption of $f_t(z) \overset{\defi}{=} \phi(\x_t^T\w)$ and $\phi(z)$ is $\gamma$-strongly convex $w.r.t.$ $\frac{1}{2}\|\cdot\|_2^2$,  it follows that
\begin{eqnarray*}
\hat{\delta}_t  &\overset{\defi}{=}& f_t(\w_{t+1})-f_t(\w_{t})-\langle f_t^{\prime}(\w_{t}), \w_{t+1}-\w_t\rangle\nonumber\\
&=&f_t(\w_{t+1})-f_t(\w_{t})- ( \phi^{\prime}_t(z)|_{z=\x_t^T\w})\cdot (\w_{t+1}-\w_t)\\
&\ge& \frac{\gamma}{2}\|\w_{t+1}-\w_t\|_2^2 = \frac{\gamma}{2}(\w_{t+1}-\w_t)^T\x_t\x_t^T(\w_{t+1}-\w_t)\label{eq:log-descrese-1}
\end{eqnarray*}
Then we have
\begin{eqnarray}
&&(f_t(\w_{t})+r(\w_{t+1}))- (f(\x_t^T \w)+r(\w))\nonumber\\
&\overset{\circlenum{1}}{\le}&\frac{1}{\eta_t}(\frac{1}{2}\|\w_t-\w\|_2^2-\frac{1}{2}\|\w_{t+1}-\w\|_2^2-\frac{1}{2}\|\w_t-\w_{t+1}\|_2^2)\nonumber\\
&&\quad-\langle f_t^{\prime}(\w_{t}), \w_{t+1}-\w_t\rangle-\hat{\delta}_t\nonumber\\
&\overset{\circlenum{2}}{\le}&\frac{1}{\eta_t}(\frac{1}{2}\|\w_t-\w\|_2^2-\frac{1}{2}\|\w_{t+1}-\w\|_2^2-\frac{1}{2}\|\w_t-\w_{t+1}\|_2^2)\nonumber\\
&& -\langle f_t^{\prime}(\w_{t}), \w_{t+1}-\w_t\rangle- \frac{\gamma}{2}(\w_{t+1}-\w_t)^T\x_t\x_t^T(\w_{t+1}-\w_t)\nonumber\\
&{=}& \frac{1}{\eta_t}(\frac{1}{2}\|\w_t-\w\|_2^2-\frac{1}{2}\|\w_{t+1}-\w\|_2^2)-\langle f_t^{\prime}(\w_{t}), \w_{t+1}-\w_t\rangle\nonumber\\
&& -   \frac{1}{2\eta_t}(\w_{t+1}-\w_t)^T(I + {\gamma\eta_t}\x_t\x_t^T)(\w_{t+1}-\w_t)\nonumber\\
&{=}& \frac{1}{\eta_t}(\frac{1}{2}\|\w_t-\w\|_2^2-\frac{1}{2}\|\w_{t+1}-\w\|_2^2)\nonumber\\
&&-\langle  (I + {\gamma\eta_t}\x_t\x_t^T)^{-1/2}f_t^{\prime}(\w_{t}), (I + {\gamma\eta_t}\x_t\x_t^T)^{1/2}(\w_{t+1}-\w_t)\rangle\nonumber\\
&& -   \frac{1}{2\eta_t}\|(I + {\gamma\eta_t}\x_t\x_t^T)^{\frac{1}{2}}(\w_{t+1}-\w_t)\|_2^2\nonumber\\
&\overset{\circlenum{3}}{\le}& \frac{1}{\eta_t}(\frac{1}{2}\|\w_t-\w\|_2^2-\frac{1}{2}\|\w_{t+1}-\w\|_2^2) + \frac{\eta_t}{2}\|(I + {\gamma\eta_t}\x_t\x_t^T)^{-1/2}f_t^{\prime}(\w_{t})\|_2^2\nonumber\\
&=&\frac{1}{\eta_t}(\frac{1}{2}\|\w_t-\w\|_2^2-\frac{1}{2}\|\w_{t+1}-\w\|_2^2) + \frac{\eta_t}{2}(f_t^{\prime}(\w_{t}))^T(I + {\gamma\eta_t}\x_t\x_t^T)^{-1}f_t^{\prime}(\w_{t})\nonumber\\
&\overset{\circlenum{4}}{\le}&\frac{1}{\eta_t}\left(\frac{1}{2}\|\w_t-\w\|_2^2-\frac{1}{2}\|\w_{t+1}-\w\|_2^2\right) + \frac{\eta_t}{2(1+\gamma\eta_t\|\x_t\|_2^2)}\| f_t^{\prime}(\x_t^T\w)\|_2^2,\nonumber \\
&\overset{\circlenum{5}}{=}&\frac{1}{\eta_t}\left(\frac{1}{2}\|\w_t-\w\|_2^2-\frac{1}{2}\|\w_{t+1}-\w\|_2^2\right) + \frac{\eta_t\|\x_t\|_2^2 (\phi^{\prime}(z))^2|_{z=\x_t^T\w_t}}{2(1+\gamma\eta_t\|\x_t\|_2^2)},\label{eq:log-descrese-1}
\end{eqnarray}
where \circlenum{1} is by Lemma \ref{lem:lem1},  \circlenum{2} is by Eq. \eqref{eq:log-descrese-1}, \circlenum{3} is by the Fenchel-Young inequality applied to $\|\cdot\|_2^2$, \circlenum{4} is by the fact $ f^{\prime}_t(\x_t^T\w)=(\phi^{\prime}_t(z)|_{z=\x_t^T\w})\cdot\x_t$ is a eigenvector of $(I + {\gamma\eta_t}\x_t\x_t^T)^{-1}$ and the corresponding eigenvalue is $\frac{1}{1+\gamma\eta_t\|\x_t\|_2^2}$, \circlenum{5} is by 
$f_t^{\prime}(\x_t^T\w)=(\phi^{\prime}_t(z)|_{z=\x_t^T\w})\cdot\x_t$.

Then summing Eq. \eqref{eq:log-descrese-1} from $t=1$ to $T$, rearranging the resulted inequality and drop out the $r(\w_{t+1})$ term, we prove the Proposition \ref{prop:1}.

\end{proof}

\end{document}